%% file: main.tex
\documentclass{article}
\usepackage{ijcai23}

\usepackage{times}
\usepackage{soul}
\usepackage{url}
\usepackage[hidelinks]{hyperref}
\usepackage[utf8]{inputenc}
\usepackage[T1]{fontenc}
\usepackage[small]{caption}
\usepackage{graphicx}
\usepackage{amsmath}
\usepackage{amsthm}
\usepackage{booktabs}
\usepackage{algorithm}
\usepackage{algorithmic}
\usepackage{macros,environments}
\usepackage{amsfonts,amssymb,latexsym,wasysym,dsfont,stmaryrd}
\usepackage{bm,xspace}
\usepackage{url} 
\usepackage[usenames,dvipsnames,svgnames]{xcolor}
\usepackage{tikz}
\usetikzlibrary{arrows,positioning}
\usepackage{booktabs}
\usepackage{multirow}
\usepackage{mathtools}

\title{\papertitle}

\author{
Munyque Mittelmann$^{1}$\and
Aniello Murano$^{1}$\And
Laurent Perrussel$^{2}$
\affiliations
$^1$University of Naples ``Federico II''\\
$^2$University of Toulouse - IRIT\\
\emails \{munyque.mittelmann, aniello.murano\}@unina.it,
laurent.perrussel@irit.fr
}

\begin{document}

\maketitle

\begin{abstract}
\input{Sections/Abstract}
\end{abstract}

\input{Sections/Intro}

\input{Sections/SLD}
\input{Sections/Games}
\input{Sections/Model_checking}
\input{Sections/Conclusion}


\appendix

\section{Model checking with memoryless strategies}
\label{apx:memoryless}
\setcounter{theorem}{0}
\setcounter{lemma}{0}
\begin{theorem}
\label{theo-mc2}
Assuming that functions in $\Func$ can be computed in polynomial space, model checking \SLD\ with 
memoryless agents is \PSPACE-complete.
\end{theorem}

\begin{proof}
    The lower bound is inherited from \SL \cite{CERMAK2018588}\footnote{The proof provided in  \cite{CERMAK2018588}(Thm. 1) considers \SL\ with epistemic operators and imperfect information, but these operators do not affect the complexity results. }, which is captured by \SLD.  
    For the upper bound, we first show that each recursive call only needs at most polynomial space. Most cases are treated analogously to the proof of Theorem 2 in \cite{SLKF_KR21}.

First, observe that each assignment $\assign$ can be stored in space $O((|\free(\phi)|+ |\Ag|)\cdot |\setpos|\cdot \log|\Act|)$. Next, for the base case, it is clear that $\semSL{\memoryless}{\assign}{\pos}{p}$ can be computed in constant space. For strategy quantification $\semSL{\memoryless}{\assign}{\pos}{\Estrata{\ag}\varphi}$, besides the recursive call to $\semSL{\memoryless}{\assign[\var\mapsto\strat]}{\pos}{\varphi}$ we need space $O(|\setpos|\cdot\log|\Act|)$ to store the current strategy and the current maximum value computed. The case for $\semSL{\memoryless}{\assign}{\pos}{(\ag,\var)\varphi}$ is clear. 
For $\semSL{\memoryless}{\assign}{\pos}{\phi_1\land\phi_2}$, we need to compute two recursive calls $\semSL{\memoryless}{\assign}{\pos}{\phi_1}$ and $\semSL{\memoryless}{\assign}{\pos}{\phi_2}$ and  compute their maximum. Similarly,  for $\semSL{\memoryless}{\assign}{\pos}{\neg \phi}$, we make one recursive call to $\semSL{\memoryless}{\assign}{\pos}{\phi}$ and subtract the resulting value from 1.  For $\semSL{\memoryless}{\assign}{\pos}{\X\phi}$, we only need to observe that the next position in $\out(\assign,\pos)$ is computed in constant space. 

We focus on the case $\semSL{\memoryless}{\assign}{\pos}{\phi_1\U_\df\phi_2}$.  Let $\iplay=\out(\pos,\assign)$. 
When evaluating a discounted operator on $\iplay$, one can restrict attention to two cases: either the satisfaction value of the formula goes below $\thresholdmc$, in which case this happens after a bounded prefix (with index $m \geq 0$), or the satisfaction value always remains above $\thresholdmc$, in which case we can replace the discounted operator with a Boolean one. This allows us to look only at a finite number of stages.

In the first case, let $m \geq 0$ denote the first index in which the satisfaction value of the formula goes below $\thresholdmc$.
\begin{align*}
    \semSL{\memoryless}{\assign}{\pos}{\phi_1\U_\df\phi_2} = & \sup_{i \ge 0}     \min 
        \Big
      (\df(i)\semSL{\memoryless}{\assign}{\iplay_{i}}{\phi_2}, 
      \\
      & \hspace{30pt} \min_{0\leq j <i} 
      \semSL{\memoryless}{\assign}{\iplay_{j}}{\phi_1}\Big)\\
      &= \max_{0\le i \le m}    \min  \Big
      (\df(j)\semSL{\memoryless}{\assign}{\iplay_{i}}{\phi_2}, 
      \\
      & \hspace{30pt}  \min_{0\leq j <i} 
      \semSL{\memoryless}{\assign}{\iplay_{j}}{\phi_1}\Big)
\end{align*}
This can be computed by a while loop that increments $i$, computes 
$\semSL{\memoryless}{\assign}{\iplay_{i}}{\phi_2}$, 
 $\min_{0\leq j <i} 
      \semSL{\memoryless}{\assign}{\iplay_{j}}{\phi_1}$ and their minimum, records the result if it is bigger than the previous maximum, and stops upon reaching a position that has already been visited. 
This requires to store the current value of $\min_{0\leq j <i} 
      \semSL{\memoryless}{\assign}{\iplay_{j}}{\phi_1}$, the current maximum, and the list of positions already visited, which are at most $|\setpos|$.
The second case is treated as for Boolean until. 

   Finally, we consider the case $\semSL{\memoryless}{\assign}{\pos}{\phi_1\until\phi_2}$.   
Notice that, since $\CGS$ has finitely many positions, there exist two indices $k<l$ such that $\iplay_k=\iplay_l$, and since strategies depend only on the current position, the suffix of $\iplay$ starting at index $l$ is equal to the suffix starting at index $k$. 
So there exist $\fplay_1=\pos_0\ldots\pos_{k-1}$ and $\fplay_2=\pos_{k}\ldots\pos_{l-1}$ such that $\iplay=\fplay_1\cdot \fplay_2^\omega$. 
It suffices to compute the prefix of $\iplay$ until the indices $l$.  
It follows that 
\begin{align*}
    \semSL{\memoryless}{\assign}{\pos}{\phi_1\U\phi_2} &= \sup_{i \ge 0}     \min 
        \Big
      (\semSL{\memoryless}{\assign}{\iplay_{i}}{\phi_2}, 
      \min_{0\leq j <i} 
      \semSL{\memoryless}{\assign}{\iplay_{j}}{\phi_1}\Big)\\
      &= \max_{0\le i \le l}    \min \Big
      (\semSL{\memoryless}{\assign}{\iplay_{i}}{\phi_2}, 
      \min_{0\leq j <i} 
      \semSL{\memoryless}{\assign}{\iplay_{j}}{\phi_1}\Big)
\end{align*}
which is computed analogously to the previous case.  

     
Next, the number of nested recursive calls is at most $|\phi|$, so the total space needed is bounded by $|\phi|$ times a polynomial in the size of the input, and is thus polynomial.
\end{proof}

\section{Model checking with perfect recall}
\label{apx:perfectrecall}

Before describing the construction of the APT, we need the following proposition, which reduces an extreme satisfaction
of an \SLD\ formula, meaning satisfaction with a value of either 0 or 1, to a Boolean satisfaction of an \SL\
formula. For the semantics of Boolean \SL, the reader may refer to \cite{MMPV14}. 
The proof proceeds by induction on the structure of the formulas.

\begin{proposition}
\label{lem:SL for positive}
Given a \CGS $\CGS$, state $\pos$, and an $\SLD$ formula $\phi$, there exist $\SL$ formulas $\posi{\phi}$ and $\notone{\phi}$ such that $|\posi{\phi}|$ and $|\notone{\phi}|$ are both $O(|\phi|)$ and the following hold for every assignment $\assign$. 
\begin{enumerate}
\item If $\semSL{\perfectrecall}{\assign}{\pos}{\phi}>0$ then $\CGS,\assign,\pos \models \posi{\phi}$, and if $\semSL{\perfectrecall}{\assign}{\pos}{\phi}<1$ then $\CGS\assign\pos\models \notone{\phi}$.
\item If $\CGS,\assign,\pos \models \posi{\phi}$ then $\semSL{\perfectrecall}{\assign}{\pos}{\phi}> 0$ and if $\CGS,\assign,\pos \models \notone{\phi}$ then $\semSL{\perfectrecall}{\assign}{\pos}{\phi}<1$.
\end{enumerate} 
\end{proposition}

Henceforth, given an $\SLD$ formula $\phi$, we refer to $\posi{\phi}$ as in Proposition~\ref{lem:SL for positive}.


Before detailing the proof for the model checking we introduce some additional definitions. For a function $f:\mathbb{N}\to [0,1]$ and for $k\in \mathbb{N}$, we define $f^{+k}:\mathbb{N}\to [0,1]$ as follows. For every $i\in \mathbb{N}$ we have that $f^{+k}(i)=f(i+k)$. 

Let $\phi$ be an $\SLD$ formula over $AP$.
We define the {\em extended closure} of $\phi$, denoted $xcl(\phi)$, to be the set of all the formulas $\psi$ of the following {\em classes}:
\begin{enumerate}
\item $\psi$ is a subformula of $\phi$.
\item $\psi$ is a subformula of $\posi{\theta}$ or $\posi{\neg \theta}$, where $\theta$ is a subformula of $\phi$. 

\item $\psi$ is of the form $\theta_1\U_{\df^{+k}} \theta_2$ for $k\in \mathbb{N}$, where $\theta_1\U_{\df}\theta_2$ is a subformula of $\phi$.
\end{enumerate}

\begin{lemma}  
\label{lemma:translation}
Let $\CGS$ be a \CGS, $\phi$ an \SLD\ formula, and  $\thresholdmc\in [0,1]$ be a threshold.  Then, there exists an 
$\auto_{\phi,\thresholdmc} = \langle \text{Val}_{\phi} \times \setpos, \setpos, Q, \delta, q_{0}, \aleph\rangle$ such that, for all states $q\in Q$, and assignments $\assign$, it holds that $\semSL{\perfectrecall}{\assign}{\pos}{\phi}>\thresholdmc$ iff $\tree \in \mathcal{L}(\auto_{\phi,\thresholdmc})$, where $\tree$ is the assignment-state encoding for $\assign$.
\end{lemma}
\begin{proof}[Proof sketch]
    The construction of the APT $\auto_{\phi,\thresholdmc}$ is done recursively on the structure of the formula $\phi$. 
    The state space $Q$ consists of two types of states. Type-1 states are assertions of the form $(\psi>t)$ or $(\psi<t)$, where $\psi\in xcl(\phi)$ is of Class 1 or 3 
    and $t\in [0,1]$. Type-2 states correspond to $\SL$ formulas of Class 2.
    Let $S$ be the set of Type-1 and Type-2 states for all $\psi\in xcl(\phi)$ and thresholds $t\in[0,1]$. Then, $Q$ is the subset of $S$ constructed on-the-fly according to the transition function defined below. We later show that $Q$ is indeed finite.

The transition function $\delta: (\text{Val}_{\phi} \times \setpos) \to {\cal B}^+(\setpos \times Q)$
is defined as follows. For Type-2 states, the transitions are as in the standard translation from $\SL$ to APT \cite{MMPV14}. For the other states, we define the transitions as follows. Let $(f, \pos)\in (\text{Val}_{\phi} \times \setpos)$, $\oplus \in \{<,>\}$, and  $\iplay=\out(\pos,\assign)$.  
\begin{itemize}

\item $\delta((true>t), (f,\pos))=
\begin{dcases} true & \text{ if }t<1\\ \textit{false} & \text{ if } t=1 \end{dcases}$

\item $\delta((\textit{false}>t), (f,\pos))=\textit{false}$
\item $\delta((true<t), (f,\pos))=\textit{false}$

\item $\delta((\textit{false}<t), (f,\pos))=\begin{dcases}true & \text{ if }t>0 \\ \textit{false} &\text{ if } t=0. \end{dcases}$

\item $\delta((p>t),(f, \pos))=\begin{dcases}
true & \text{ if }p\in \val(\pos) \text{ and } t<1 \\
\textit{false} & \text{ otherwise.}
\end{dcases}$

\item
$\delta((p<t),(f, \pos))=\begin{dcases}
\textit{false} & \text{ if }p\in \val(\pos) \text { or } t=0,\\
true & \text{ otherwise.}
\end{dcases}$

\item $\delta((\psi_1\vee \psi_2 \oplus t),(f,\pos))=\delta((\psi_1 \oplus t),(f,\pos))\vee \delta( (\psi_2 \oplus t),(f,\pos))$ 


\item $\delta((\exists\var \psi) \oplus t), (f, \pos)) = \bigvee_{\act \in \Act} \delta'( \psi \oplus t, (f[\var \mapsto \act], \pos))$  
where $\delta_\psi'$ is obtained by  nondeterminizing the APT $\auto_{\psi,t}$, by applying the classic transformation \cite{MULLER1987267} which gives the equivalent NPT $\auton_{\psi, t} = \langle \text{Val}_{\psi} \times \setpos, \setpos, Q', \delta', q_{0}', \aleph'\rangle$

\item $\delta(((\var,\ag)\psi \oplus t),(f, \pos)) = 
\delta'( (\psi \oplus t), (f', \pos))$ where $f' = f[\varb \mapsto f(\var)]$ if $\varb \in \text{free}(\psi)$, and $f'=f$ otherwise
 
\item $\delta((\neg\psi \oplus t),(f, \pos)) = \delta'((\psi \oplus t), (f, \pos))$  where $\delta'$  is obtained by dualizing the automaton $\auto_{\psi,t} $ \cite{MULLER1987267}, which gives the automata $\bar{\auto}_{\psi,t}=\langle \text{Val}_{\psi} \times \setpos, \setpos, Q', \delta', q_{0}', \aleph'\rangle$


\item $\delta((\X \psi_1>t),(f,\pos))=\delta((\psi_1>t),(f,\iplay_0))$ 
 \item $\delta((\X \psi_1<t),(f,\pos))=\delta((\psi_1<t),(f,\iplay_0))$

\item $\delta((\psi_1\U \psi_2>t),(f,\pos))\!=\!
\begin{dcases}
\delta_> & \text{ if } 0<t< 1\\ \textit{false} &  \text{ if } t \geq 1\\ \delta_0 & \text{ if } t=0 \end{dcases}  $

where $\delta_> = \delta((\psi_2> t),(f,\pos))\vee [\delta((\psi_1>t),(f,\pos))\wedge (\psi_1\U \psi_2> t)]$ and $\delta_0= \delta((\posi{(\psi_1\U \psi_2)}),(f,\pos))$

\item $\delta((\psi_1\U \psi_2<t),(f,\pos))\!=\!\begin{dcases} \delta'& \text{ if } 0<t\le 1\\ true &  \text{ if } t > 1\\ \textit{false} & \text{ if }  t=0 \end{dcases} $
where $\delta_< = \delta((\psi_2< t),(f,\pos))\wedge [\delta((\psi_1<t),(f,\pos))\vee (\psi_1\U \psi_2< t)] $

\item $\delta((\psi_1\U_\df \psi_2>t),(f,\pos))\!=\!\begin{dcases} \delta_> & \text{ if } 0<\frac{t}{\df(0)}< 1\\ \textit{false} &  \text{ if } \frac{t}{\df(0)} \geq 1\\ \delta_0 & \text{ if } \frac{t}{\df(0)}=0 \end{dcases} $

where $\delta_>  = \delta((\psi_2> \frac{t}{\df(0)}),(f,\pos))\vee [\delta((\psi_1>\frac{t}{\df(0)}),\allowbreak (f,\pos)) \allowbreak \wedge (\psi_1\U_{\df^{+1}}\psi_2> t)]$ and $\delta((\posi{(\psi_1\U_\df \psi_2)}),(f,\pos))$ 

\item$\delta((\psi_1\U_\df \psi_2<t),(f,\pos))\!=\!  \begin{dcases} \delta_< & \text{ if } 0<\frac{t}{\df(0)}\le 1\\ true &  \text{ if } \frac{t}{\df(0)}> 1\\ \textit{false} & \text{ if } \frac{t}{\df(0)}=0 
\end{dcases} $ 

where $\delta_< = \delta((\psi_2<\frac{t}{\df(0)}),(f,\pos))\wedge  [\delta((\psi_1<\frac{t}{\df(0)}),(f,\pos))\vee (\psi_1\U_{\df^{+1}}\psi_2<t)]$
\end{itemize} 


The initial state of $\auto_{\phi,\thresholdmc}$ is $(\phi>\thresholdmc)$. 
The accepting states are these of the form $(\psi_1 \U \psi_2 < t)$, as well as accepting states that arise in the standard translation of Boolean \SL to APT (in Type-2 states).  
While the construction as described above is infinite (indeed, uncountable), only finitely many states are reachable from the initial state, and we can compute these
states in advance.  
This follows from the fact that once the proportion between $t$ and $\df(i)$ goes
above 1, for Type-1 states associated with threshold $t$ and sub formulas with a discounting function $\df$, we do not have to generate new states.
\end{proof}
\section*{Acknowledgments}
We thank the ANR project AGAPE ANR-18-CE23-0013,  the PNNR FAIR project, the InDAM project ``Strategic Reasoning in Mechanism Design'', the PRIN 2020 Project RIPER, and the EU ICT-48 2020 project TAILOR (No. 952215).  
\bibliographystyle{named}

\end{document}

%% file: Sections/Abstract.tex
Discounting is an important dimension in multi-agent systems as long as we want to reason about strategies and time. It is a key aspect in economics as it captures the intuition that the far-away future is not as important as the near future. 
Traditional verification techniques allow to check whether there is a winning strategy for a group of agents but they do not take into account the fact that satisfying a goal sooner is different from satisfying it after a long wait. 
In this paper, we augment Strategy Logic with \emph{future discounting} over a set of discounted functions \Func, denoted \SLD. We consider “until” operators with discounting functions: the satisfaction value of a specification in \SLD\ is a value in $[0, 1]$, where the longer it takes to fulfill requirements, the smaller the satisfaction value is. We motivate our approach with classical  examples from Game Theory  and  study the complexity of model-checking  \SLD-formulas.

%% file: Sections/Intro.tex
\section{Introduction}
\label{sec:intro}

The goal of this paper is to advance the research on strategic reasoning and formal verification by considering a discounting effect: the utility of agents decreases over time.
Boolean state-transition models 
have been widely used to define the semantics of  temporal and strategic logics, including Linear Temporal Logic (\LTL) \cite{pnueli1977temporal},  Alternating-time Temporal Logic (\ATL) \cite{AlurHK02}, Strategy Logic (\SL) \cite{MMPV14,ChatterjeeHP10}.  
In conjunction with model checking techniques  \cite{clarke2018handbook}, these formal frameworks are useful for the representation and verification of  hardware and software systems. 
Given a strategic logic specification, 
the correctness of a system is a yes/no matter: either the system satisfies the specification or it does not. Complex systems that interact with a physical environment or that are composed of multiple autonomous agents 
may have quantitative aspects described by real numbers (e.g. utilities, time and costs). Evaluating the \emph{quality} of such systems through the Boolean satisfaction of the specifications is often inadequate. Different 
levels of quality may exist, and this should be reflected in the output of the verification procedure \cite{AlmagorBK14}. 

In this work, we are interested in verifying Multi-Agent Systems (MAS) whose quality assessment needs to take into account that satisfying the goal sooner is different from satisfying it after a long wait. To illustrate this setting, 
consider an agent whose task is to organize a trip and who is facing the problem of booking a flight. An early booking is more susceptible to becoming unfeasible in the case of unforeseen changes in the travel plans. On the other hand, waiting to book may result in 
more important costs 
for the agent. Moreover, the trip-organizing agent may be a part of a system composed of other, self-interested, agents. In this case, the agents' interactions can also influence their ability to find reasonable flight options and price tags. On one side, there is a competitive aspect when agents dispute the last available tickets. Cooperation could also take place as some companies offer discounts for group booking. 
To address this problem for (single-agent) systems, researchers have suggested to augment Linear Temporal Logic with future discounting \cite{henzinger2005,AlmagorBK14}. In the discounted setting, the satisfaction value of specifications is a numerical value, and it depends, according to some discounting function, on the time waited for eventualities to get satisfied. 

Discounting is a key dimension in Economics and has been studied in Markov decision processes \cite{Jerzy1996} as well as game theory \cite{shapley1953stochastic} and system theory \cite{alfaro2003discounting} to capture the  intuition that the far-away future is not as important as the near future. The multi-agent setting 
has also been widely investigated, including repeated games 
\cite{abreu1988theory,fudenberg2009folk,pkeski2014repeated}, 
the prisoner's dilemma game \cite{harris2002delay,locey2013social}, 
and negotiation protocols \cite{weg1990two,
FatimaWJ06}, to name a few.   
Previous work \cite{jamroga2008temporal,DBChenFKPS13} have initiated to study logics inspired on \ATL and Markov chains for reasoning about discounting in stochastic MAS. Likewise \ATL, these logics are unable to capture complex solution concepts in MAS (such as Nash equilibria), which are important when evaluating the possible outcomes of such systems.

\paragraph{Contribution.}  
In this work, we augment Strategy Logic with future discounting, denoted \SLD, and study its complexity for model-checking.  
The main advantage of this logic is that it allows us to express and verify (i) the strategic abilities of agents to achieve certain goals while considering temporal discounts,  
and (ii) complex strategy concepts such as Nash equilibrium of discounted games. %
Different from previous work, we focus on deterministic games   
and consider temporal discounting alongside a logic that
quantifies over strategies. This enables an unbounded number of alternations from strategic operators which is necessary to capture complex solution concepts. In relation to technical results, we also studied the complexity of the model-checking problem under memoryless and perfect recall strategies, which was not established in \cite{jamroga2008temporal}. 

\SLD\ represents a family of logics, each one parameterized by a set of discounting functions. Considering a set of functions allows us to model games in which each agent, or a coalition of them, is affected differently by how long in the future events occur (\eg, patient vs hurried agents). 
 We also provide complexity results for model-checking and motivate the approach with classical examples from Game Theory.

This is the first work to consider a Strategy Logic with discounting for strategic reasoning in MAS. 
We aim at paving the way for a new line of research 
that applies the formal techniques developed for verification and reasoning in MAS  to game-theoretic problems involving future discounts.

\paragraph{Outline.} The paper\footnote{This paper is an extended version of a paper accepted to the 32nd International Joint Conference in Artificial Intelligence (IJCAI 2023). } is organized as follows:  we start by discussing related work in Section \ref{sec:rw}. Then, we define Strategy Logic with future discounts, denoted \SLD\ (Section \ref{sec:SL}). We proceed by introducing  problems and concepts on using  discounting in multi-agent games and illustrate the use of \SLD\ (Section \ref{sec:games}). Next, we study the complexity results for model checking (Section \ref{sec:mc}). 
Finally, we conclude the paper and point directions for future work (Section \ref{sec:conclusion}). 

\section{Related Work} 
\label{sec:rw}
Weighted games have been studied in the literature in relation to various kinds of objectives, including  parity~\cite{EJ91},
mean-payoff~\cite{ehrenfeucht1979positional,ZwickP96},  
energy~\cite{CAHS03,BouyerFLMS08}, and combining qualitative and quantitative objectives  in equilibrium \cite{GutierrezMPRSW21}.
 \SLF~\cite{Bouyer19,10.1145/3582498} was recently introduced as a quantitative extension of \SL defined over weighted concurrent game structures. It extends \LTLF~\cite{almagor2016formally}, a multi-valued logic that augments \LTL with quality operators. \SLF subsumes both \SL and \LTLF and is expressive enough to express complex solution concepts such as Nash equilibrium and properties about quantities. An extension of \SLF with imperfect information and epistemic operators was recently proposed \cite{SLKF_KR21}.  Other quantitative extensions of  \LTL have been explored in the context of  averaging~\cite{BouyerMM14},  discounting
~\cite{AlmagorBK14,Mandrali12}, and mean-payoff objectives~\cite{BohyBFR13}. 
Quantitative extensions of \ATL have also been investigated, such as 
timed \ATL~\cite{HP06,BLMO07}, multi-valued~\ATL~\cite{JamrogaKKP20}, 
\ATL with resource bounds~\cite{AlechinaLNR17,AlechinaBDL18},
and
weighted versions of \ATL~\cite{LMO06,BullingG22,Ves15}. Another related  problem is prompt requirements (see, for instance, \cite{DBLP:conf/kr/AminofMRZ16,DBLP:conf/ijcai/FijalkowMMV20}), which consider a bound on the number of steps to satisfy the specification. 

To encode the notion that the importance of events should be discounted according to how late they occur, 
\citeauthor{henzinger2005} (\citeyear{henzinger2005}) proposed an extension of the 
 Computational Tree Logic with quantitative semantics. In this logic,  path operators
are discounted by a parameter that can be chosen to give more weight to states that are closer to the beginning of the path.
Later,  \citeauthor{AlmagorBK14} \citeyear{AlmagorBK14}  proposed \LTL augmented with an arbitrary set of discounting functions, denoted \LTLD; and further explored with unary propositional quality operators and average-operator. 

In the context of stochastic systems, \citeauthor{Jamroga08AAMAS} (\citeyear{Jamroga08AAMAS}) proposed the Markov Temporal Logic, which extends the Branching-Time Temporal Logic and captures discounted goals. Later, this approach was extended to the multi-agent setting \cite{jamroga2008temporal}. 
Finally, \citeauthor{DBChenFKPS13} (\citeyear{DBChenFKPS13}) considered a probabilistic extension of \ATL, alongside discounted rewards. 

Temporal and strategic logics  have been successfully applied alongside model-checking techniques to the certification of several types of MAS, such as voting protocols  \cite{BelardinelliCDJ17,JamrogaKM20}, autonomous robotic systems \cite{luckcuck2019formal},  smart contracts \cite{tolmach2021survey}, avionic systems \cite{elkholy2020model}, and task coordination robots \cite{lacerda2019petri}.

%% file: Sections/SLD.tex
\section{Strategy Logic With Discounting}
\label{sec:SL}

Strategy Logic with Discount ($\SLD$) generalizes $\SL$ by adding discounting temporal operators. The logic is actually a family of logics, each parameterized by a set $\Func$ of discounting functions.
A function $\df: \setn\to [0,1]$ is a \emph{discounting function} if $\lim_{i\to \infty}\df(i)=\limit$, and $\df$ is 
non-increasing. Examples of  discounting functions include $\df(i)=\lambda^i$, for some $\lambda\in (0,1)$, and $\df(i)=\frac{1}{i+1}$. 

For the remainder of the paper, we fix a set of discounting functions $\Func$, a set of atomic propositions $\APf$, a set of agents $\Ag$, and a set of strategy variables $\Var$, except when stated otherwise. We let $\AgSize$ be the number of agents in $\Ag$.

The syntax of $\SLD$ adds to $\SL$ the operator $\phi\until_{\! \df} \psi$ (discounting-Until), for every function $\df \in \Func$. The logic is defined $\SLD$ as follows:  
  
\begin{definition}
  The syntax of \SLD\  is defined by the grammar
\begin{align*}
  	\varphi  ::= p \mid \neg \varphi \mid \varphi \lor \varphi \mid \Estrata{} 
  	\varphi \mid (\ag, \var)  \varphi  
        &\mid \X \phi \mid \phi \U \phi   \mid \phi \U_{\! \df} \phi  
\end{align*}
where $p\in\APf$,  $\var 
\in\Var$, $\ag\in\Ag$, and $\df\in \Func$. 
\end{definition}

The intuitive reading of the operators is as follows: $\Estrata{}\varphi$ means that there exists a strategy such that $\phi$ holds; $(\ag,\var)\varphi$ means that when strategy $\var$ is assigned (or ``bound'') to agent $\ag$, $\phi$ holds;  
$\X$ and $\U$ are the usual temporal operators ``next'' and ``until''. 
The intuition of the operator $\U_{\! \df}$ is that events that happen in the future have a lower influence, and the rate by which this influence decreases depends on the function $\df$.  .   

A variable is \emph{free} in a formula $\phi$ if it is bound to an agent without being quantified upon, and an agent $\ag$ is free in $\phi$ if $\phi$ contains a temporal operator ($\X$, $\U$, $\U_{\! \df}$)  not in the scope of any binding for $\ag$. The set of free variables and agents in $\phi$ is written $\free(\phi)$, and a formula $\phi$ is a \emph{sentence} if $\free(\phi)=\emptyset$. 

A 
state-transition model
is a labeled directed graph, 
in which the vertices represent the system states, the edges the state changes (\eg, according to environment or agents' actions), and the labels the Boolean characteristics of the state (\ie, the truth values of state atomic propositions). 
In this paper, we consider state-transition models in which there are multiple agents that act simultaneously and independently. These models are called concurrent game structures (\CGS). 

\begin{definition}
  \label{def-wcgs}
A \emph{concurrent game structure} (\CGS) is a tuple
$\wCGS=(\Act,\setpos,\pos_\init,\trans,\val)$ where 
(i) $\Act$ is a  finite set of \emph{actions};     (ii) $\setpos$ is a  finite set of \emph{positions};
(iii) $\pos_\init\in\setpos$ is an \emph{initial position};  
(iv)    $\trans:\setpos\times \Act^\Ag\to\setpos$    is a \emph{transition function};  
(v)    $\val:\setpos\to 2^\APf$ is a \emph{labeling function}.  
\end{definition}

In a position $\pos\in\setpos$, each player $\ag$ chooses an action $\mova\in\Mov$, and the game proceeds to position $\trans(\pos, \jmov)$ where $\jmov \in \Act^\Ag$ is an \emph{action profile}
$(\mova)_{\ag\in\Ag}$.  

We write $\profile{o}$ for a tuple of objects $(o_\ag)_{\ag\in\Ag}$, one for each agent, and such tuples are called \emph{profiles}.
Given a profile $\profile{o}$ and $\ag\in\Ag$, we let $o_\ag$ be agent $\ag$'s component, and $\profile{o}_{-\ag}$ is $(o_\agb)_{\agb\neq\ag}$. Similarly, we let $\Ag_{-\ag}=\Ag\setminus\{\ag\}$. 
For a group of $n$ agents $A = \{\ag_1, \dots, \ag_n\}$ and strategy profile $\sigma = \sigma_1, \dots, \sigma_n$ we write $(A,\sigma)$ as a shortcut for $(a_1,\sigma_1) \dots (a_n, \sigma_n)$.

A \emph{play} $\iplay=\pos_0\pos_1\ldots$ in $\wCGS$ is an infinite sequence of positions such that $\pos_0=\pos_\init$ and for every $i\geq 0$ there exists an action profile $\jmov$ such that $\trans(\pos_{i}, \jmov)=\pos_{i+1}$. We write $\iplay_i=\pos_i$ for the position at index $i$ in play $\iplay$. A \emph{history} $\hist$ is a finite prefix of a play, $\last(\hist)$ is the last position of history $\hist$, $|\hist|$ is the length of $\hist$ and $\Hist$ is the set of histories.


A (perfect recall) \emph{strategy} is a  function $\strat:\Hist
\to \Mov$ that maps each history to an action.   
A (memoryless) \emph{strategy} is a  function $\strat:\setpos \to \Mov$ that maps each position to an action. 
We let $\setstrat^\perfectrecall$ (similarly $\setstrat^\memoryless$) be the set of perfect recall strategies  (resp. memoryless strategies).  
For the remainder of the paper, we use $\memoryless$ and $\perfectrecall$ to denote memoryless and perfect recall, respectively, and we let $\setting = \{\memoryless, \perfectrecall\}$.

An \emph{assignment}  $\assign:\Ag\union\Var \to \setstrat$ is
a function from players and variables to strategies.
For an assignment
$\assign$, an agent $\ag$ and a strategy $\strat$ for $\ag$,
$\assign[a\mapsto\strat]$ is the assignment that maps $a$ to $\strat$ and is otherwise equal to
$\assign$, and 
$\assign[\var\mapsto \strat]$ is defined similarly, where $\var$ is a
variable. 

For an assignment $\assign$ and a state $\pos$, we let
$\out(\assign,\pos)$ be the unique play that continues $\pos$ following the strategies
assigned by $\assign$. Formally,
 $\out(\assign,\pos)$ is the play $\pos\pos_{0}\pos_{1}\ldots$ such that for all $i\geq 0$, $\pos_{i}=\trans(\pos_{i-1},\jmov)$ where for all $\ag\in\Ag$, $\jmov_\ag=\assign(\ag)(\pos\pos_{1}\ldots\pos_{i-1})$. 
 \begin{definition}	
 \label{def:semSL}
  Let $\CGS=(\Act,\setpos,\pos_\init,\trans,\val)$ be a \CGS, 
   $\assign$ be an assignment, and  $\setting\in \{\perfectrecall,\memoryless\}$. The satisfaction value $\semSL{\setting}{\assign}{\pos}{\varphi}\in[\lowb,1]$ of  an \SLD\ formula $\phi$ in a state~$\pos$ 
  is defined as follows, where $\iplay$ denotes $\out(\assign,\pos)$: %
  \begingroup
  \allowdisplaybreaks
  \begin{align*}
    \semSL{\setting}{\assign}{\pos}{p} & =
    \begin{cases}
   1        & \text{if } p \in \val(\pos)\\ 
   0        & \text{otherwise}
  \end{cases}
  \\
    \semSL{\setting}{\assign}{\pos}{\Estrat\varphi} &=  \max_{\strat \in \setstrat}
    \semSL{\setting}{\assign[\var \mapsto \strat]}{\pos}{\varphi} \\
    \semSL{\setting}{\assign}{\pos}{(\ag,\var)\varphi} & = \semSL{\setting}{\assign[\ag \mapsto
      \assign(\var)]}{\pos}{\varphi} \\ 
     \semSL{\setting}{\assign}{\pos}{\phi_1 \lor \phi_2} &= \max( \semSL{\setting}{\assign}{\pos}{\phi_1},  \semSL{\setting}{\assign}{\pos}{\phi_2})\\
       \semSL{\setting}{\assign}{\pos}{\neg \phi} &= 1-\semSL{\setting}{\assign}{\pos}{\phi}\\
    \semSL{\setting}{\assign}{\pos}{\X \phi} &= \semSL{\setting}{\assign}{\iplay_{1}}
    {\phi} \\
    \semSL{\setting}{\assign}{\pos}{\phi_1\U\phi_2} & = \sup_{i \ge 0}     \min\! 
       \big (\semSL{\setting}{\assign}{\iplay_{i}}{\phi_2},\\[-0.75em]
       & \hspace{2.5cm} 
      \min_{0\leq j <i}       \semSL{\setting}{\assign}{\iplay_{j}}{\phi_1} \big) \\\\[0.25em]       \semSL{\setting}{\assign}{\pos}{\phi_1\U_{\! \df} \phi_2} & = \sup_{i \ge 0}     \min\!         \big (d(i)\semSL{\setting}{\assign}{\iplay_{i}}{\phi_2},\\[-0.75em]
      & \hspace{2cm}
      \min_{0\leq j <i}       d(j)\semSL{\setting}{\assign}{\iplay_{j}}{\phi_1} \big) 
  \end{align*}
  \endgroup
\end{definition}
If $\phi$ is a sentence, its satisfaction value does not depend on the assignment, and we write $\semSL{\setting}{}{\pos}{\phi}$ for $\semSL{\setting}{\assign}{\pos}{\phi}$ where $\assign$ is any assignment. We also let $\semSLglobal{\setting}{\phi}=\semSL{\setting}{}{\pos_\init}{\phi}$.

Classical abbreviations are defined as follows: 
${\perp\colonequals\neg\top}$, ${\varphi\wedge\varphi'
\colonequals \neg (\neg \varphi \vee \neg \varphi')}$, ${\varphi\rightarrow \varphi' \colonequals \neg
\varphi \vee \varphi'}$, ${\F\psi \colonequals \top \U \psi}$,
${\G\psi \colonequals \neg \F \neg \psi}$ and ${\Astrat\varphi\colonequals \neg\Estrat\neg\varphi}$. The quantitative counterparts of  $\G$ and $\F$, denoted $\G_{\! \df}$ and $\F_{\! \df}$, are defined analogously. 

\begin{remark}     Since we consider discounting functions, the satisfaction value of future events in formulas involving discount functions tends to $\limit$.  \end{remark}

\paragraph{Relation with \LTLD, \SL and \SLF.}
\LTLD\ \cite{AlmagorBK14} is the fragment of \SLD\ without strategy quantification and bindings. 
Considering that the satisfactions values  $1$ and $0$ represent $true$ and $\textit{false}$ (resp.), \SL \cite{MMPV14} is a syntactical restriction of \SLD\  (without the discounted-Until). 
\SL\ cannot express that the value of a formula decays over time. 
To notice the difference, assume a CGS $\CGS$, an assignment $\assign$ and states $\pos, \pos'$ such that %
and $\out(\assign, \pos') = \pi_0\pi$ and $\pi = \out(\assign, \pos)$, that is, the outcome from $\pos'$  is the outcome from $\pos$ with the first state repeated. 
Assuming that $p \in \val(\pi_i)$ and $\df(i) \neq \df(i-1)$,  for some $i\geq 1$ and $p \in \Prop$, 
we have that $\semSL{\setting}{\assign}{\pos}{\F_{\!\df} \, p} \neq \semSL{\setting}{\assign}{\pi'}{\F_{\!\df} \, p}$.  
However, using the classical until, we have that $\semSL{\setting}{\assign}{\pos}{\F \, p} = \semSL{\setting}{\assign}{\pi'}{\F \, p}$. 
As for 
\SLF  \cite{Bouyer19}, notice it is interpreted over different classes of models from \SLD, namely weighted \CGS, which uses weight functions for atomic propositions in place of propositional labeling of states. 
\SLF is defined over a set of functions $\FuncF$ over $[0, 1]$, but its semantics does not enable to use these functions to capture the effect of future discounts. This is because  
functions are applied over the satisfaction value of formulas in a given state, independent from \emph{how far} in the play they are being evaluated w.r.t. the initial state.

%% file: Sections/Games.tex
\section{Discounting in Multi-Agent Games}
\label{sec:games}
We now introduce problems and concepts from Game Theory that motivated reasoning about discounts in MAS. 

\subsection{Nash Equilibrium for \SLD\ Goals}
\emph{Nash equilibrium} (NE) is a central solution concept in game theory that captures
the notion of a stable solution, that is a solution from which no single player can
individually improve his or her welfare by deviating \cite{Nisan2007}. 
Deterministic concurrent multi-player Nash equilibrium can be expressed using  \SL (or its extensions) for Boolean valued goals \cite{MMPV14} and quantitative goals \cite{Bouyer19}. With \SLD, we can express that agent's goals are affected by how long in the future they are achieved. 
  
		Let the \LTLD-formula $\psi_{\ag}$ (\ie, an \SLD\ formula without bindings and strategy quantification) denote the goal of agent $\ag$.
			We can express whether a strategy profile $\profile \strat = (\strat_\ag)_{\ag \in \Ag}$ 
   is a \emph{Nash equilibrium}   through the \SLD\ formula $$\varphi_{\text{NE}}(\profile \strat) \egdef 
   (\Ag,\profile \strat)\bigwedge_{\ag \in \Ag} \big(\Astrat[\varb]  (\ag, \varb) \psi_\ag \big) \to \psi_\ag $$

   The existence of a Nash equilibrium is captured by the formula $\hat{\varphi}_{\text{NE}} \egdef 
   \exists\profile{\strat} ( \varphi_{\text{NE}}(\profile \strat))$.  
This is a classical problem in game theory and, more precisely when studying games with future discounting \cite{FUDENBERG1990194}. 

As we shall see in the next sections, the goal $\psi_\ag$ of an agent $\ag$ may involve temporal discounts. 
In the booking agent example, for instance, the discounted goal $$\psi_\ag \egdef \text{priceunder}_\vartheta \U_\df \text{booked}_\ag$$ specifies that the flight ticket is affordable (that is, below a threshold $\vartheta$) until agent $\ag$ booked her ticket. The value obtained from achieving the goal later is reduced according to the discounted function $\df$.

\subsection{Secretary Problem}
The classical secretary problem studies the problem of an agent selecting online an element (called a “secretary”) the \emph{maximum value} from a known number of candidates to be presented one by one in random order.  As each item is presented she must either accept it, in which case the game ends, or reject it. In the second case, the next item in the sequence is presented and the agent faces the same choice as before \cite{freeman1983secretary}. 
Applications of this problem include agents' facing the decision of buying a house or  hiring employees. 
Several variants and extensions of the secretary problem are considered in the literature, including using time-dependent discount factors to reduce the benefit derived from selecting a secretary at a later time \cite{babaioff2009secretary}. 
The discounted setting captures the cost of rejecting elements. For instance, when seeking to purchase a house, an agent may prefer to  chose a  suboptimal house at the beginning of the game  than wait longer  to pick her most desirable house.
Recently, \citeauthor{DoHL022} (\citeyear{DoHL022}) investigated the selection of $k$ secretaries by a multi-agent selection committee. The hiring decision is made by a group of  voting agents that  specify whether they consider acceptable to hire the current candidate or not. 

With \CGS, we can represent deterministic
perfect information instances of the secretary problem. Let us consider the selection of $k$ secretaries by multiple voting agents. 
For each candidate $j$ from a finite set of candidates $C$, we let the atomic propositions $\text{present}_j$ denote whether  she was presented and $\text{hired}_j$ denote whether she was hired. Proposition $k\text{-hired}$ specifies whether $k$ secretaries were already selected\footnote{  
The formalization of the game as a \CGS is left to the reader. Examples on how to model  similar problems using  \CGS can be found in  \cite{mittelmann2020auction,SLKF_KR21}.}. 

The \SLD\ formula $\F_{\!\df} \, k\text{-hired}$ represents the goal of having enough candidates hired in the future. The satisfaction value of this goal decreases according to $\df$, denoting that it is preferable to hire $k$ candidates as soon as possible.  
The discounted goal 
 $$\exists\var \forall \profile\varb (\ag, \var)(\Ag_{-\ag}, \profile \varb)(\bigvee_{j\in C} \neg \text{present}_j) \U_{\! \df} \, k\text{-hired}$$
 represents that the voter $\ag$ has a strategy to ensure  that, no matter the strategies of the other agents, there are candidates still not presented until enough secretaries were hired.


In Figure \ref{fig:sec}, we exemplify the \CGS $\CGS_{sec}$ representing an instance of the secretary problem with  two voting agents, Ann and Bob, and three candidates, $a$, $b$, and $c$. In  the initial state ($q_0$), the agents vote on whether they want to hire candidate $a$ by performing the action $y$ or $n$. 
Candidate $a$ is hired only if both agents play $y$, in which case the game moves to state $q_2$. Otherwise, the game proceeds to state $q_1$ in which they can vote for candidate $b$ (and similarly, for candidate $c$ in state $q_3$. 
 The game ends when one secretary is hired (states $q_2$, $q_4$, and $q_6$) or all candidates have been presented (state $q_5$). 
 
\input{Sections/figure-secretary}
 
We let the following \SLD\ formulas denote agent $a$'s and agent $b$'s goals, resp.:
$$\psi_{\text{Ann}} \egdef \F \, \text{hired}_b \lor \F_{\!\df_{\text{Ann}}} \, 1\text{-hired} $$
$$\psi_{\text{Bob}} \egdef \F_{\! \df_{\text{Bob}}} \, 1\text{-hired}$$
and we assume the  discount functions $\df_{\text{Ann}}(i)=\frac{1}{i+1}$ and $\df_{\text{Bob}}(i)=(\frac{1}{2})^i$. 
In other words, Ann's goal is to hire candidate $b$ in the future or to hire any candidate (with a discount according to $\df_{\text{Ann}}$), while Bob's goal is to hire a candidate in the future (with a discount given by $\df_{\text{Bob}}$). 
Notice that without the discount functions, hiring a secretary earlier would be similar to hiring  later. 
The two discount functions stress that Bob is more eager to hire a secretary than Ann. 
Table \ref{tab:discount} shows the value of the functions in each time $i$.

\begin{table}[h]
\centering
\begin{tabular}{*5c}  
\toprule
$i$ & 0 & 1 & 2 & 3 \\  
\midrule
$\df_{\text{Ann}}$ & 1 & 0.5 & 0.333 & 0.25 \\  
$\df_{\text{Bob}}$ &  1 & 0.5 & 0.25 & 0.125 \\ 
\bottomrule
\end{tabular}
\caption{Values for $\df_{\text{Ann}}(i)$ and $\df_{\text{Ann}}(i)$}
\label{tab:discount}
\end{table}

The satisfaction value of the agents' goals is only different from 0 in the states in which a candidate were hired. 
Let $\strat_{abc}$ denote the strategy of playing $y$ for each candidate (that is, $\strat_{abc}(q_0)=\strat_{abc}(q_1)=\strat_{abc}(q_3)=y)$, $\strat_{bc}$ denote the strategy of playing $y$ only for candidates $b$ and $c$, and  $\strat_{c}$ denote the strategy of playing $y$ only for $c$. Table \ref{tab:goalsec} shows the satisfaction value of agents goals' from the initial state $q_0$ for different assignments of strategies. 


\begin{table}[h]
\centering
\small
\begin{tabular}{ccccc}
\toprule
 &  & \multicolumn{3}{c}{$\assign(\text{Bob})$} \\ 
\multicolumn{1}{c}{} & \multicolumn{1}{c}{} & \multicolumn{1}{c}{$\strat_{abc}$} & \multicolumn{1}{c}{$\strat_{bc}$} & \multicolumn{1}{c}{$\strat_{c}$}  \\ \cmidrule(r){3-5}
\multicolumn{1}{c}{\multirow{3}{*}{$\assign(\text{Ann})$}} & \multicolumn{1}{l|}{$\strat_{abc}$} & \multicolumn{1}{c}{(0.5, 0.5)} & \multicolumn{1}{c}{(1, 0.25)} & \multicolumn{1}{c}{(0.25, 0.125)} \\ 
\multicolumn{1}{c}{} & \multicolumn{1}{l|}{$\strat_{bc}$} & \multicolumn{1}{c}{(1, 0.25)} & \multicolumn{1}{c}{(1, 0.25)} & \multicolumn{1}{c}{(0.25, 0.125)} \\ 
\multicolumn{1}{c}{} & \multicolumn{1}{l|}{$\strat_{c}$} & \multicolumn{1}{c}{(0.25, 0.125)} & \multicolumn{1}{c}{(0.25, 0.125)} & \multicolumn{1}{c}{(0.25, 0.125)} \\ 
\bottomrule
\end{tabular}
\caption{Value of $(\semSLsec{\assign}{q_0}{\psi_{\text{Ann}}}, \semSLsec{\assign}{q_0}{\psi_{\text{Bob}}})$ for different strategy assignments $\assign$.}
\label{tab:goalsec}
\end{table}



As illustrated on Table \ref{tab:goalsec},   the strategy profile $(\strat_{bc},\strat_{abc})$ is a Nash equilibrium and thus 
$\semSLsecglobal{\hat{\varphi}_{\text{NE}}}\neq 0$  
(note memoryless strategies are enough for this problem).

\subsection{Negotiation With Time Constraints}
Let us consider a second context where discounting is a key issue. Negotiation is a type of interaction in MAS in which disputing agents decide how to divide a resource. 
Time constraints, which may be in the form of both deadlines and discount factors, are an essential element of negotiation because the interaction cannot go on indefinitely and must end within a reasonable time limit \cite{livne1979role}. 
Here we consider the problem of 
negotiation with time constraints studied in \cite{rubinstein1982perfect,FatimaWJ06}, and generalize to the multiple agent case. In this problem, agents
want to determine how to divide (single or multiple) issues, called  ``pies'', of size $1$ among themselves. 
The negotiation must end in at most $n \in \setn^+$ rounds. This deadline can be represented with an arbitrary discounting function $\df_{n}$ such that $\df_{n}(n) = 0$. In this case, a goal in the form $\F_{\! \df_n} \, \psi$ motivates agents to achieve $\psi$ before the $n$-th stage of the negotiation.  

The negotiation process is made by alternating offers from the agents. Initially, $\ag$ starts by making an offer on how to divide a pie 
to the other agents $\Ag_{-\ag}$.
Agents in $\Ag_{-\ag}$ can either accept or reject this offer.  
If agents in $\Ag_{-\ag}$ accept, the negotiation ends in an agreement with the proposed share. 
Otherwise, an agent $\agb \neq \ag$
makes a counteroffer in the next round. The negotiation proceeds until there is an agreement on accepting an offer. 
The key feature of this problem is that the pie is assumed to shrink (\ie, to lose value) with time \cite{rubinstein1982perfect}. This represents the situation in which the pie perishes with time or is affected by  inflation.  The pie shrinkage is represented with by a discount function $\df_{pie}$. 
At time $i = 1$, the size of the pie is $1$, but in all subsequent time periods $i > 1$, the pie shrinks to $\df_{pie}(i)$. 

Figure \ref{fig:neg} shows the \CGS $\CGS_{ngt}$, which illustrates an instance of the negotiation problem with a single-issue and two agents, Alice and Beth. The game starts in state $q_0$, where Alice can make an offer to split the pie either so as to take half or two thirds of it for herself (while the remaining of the pie is left for Beth).   In the next state (either $q_1$ or $q_2$, according to Alice's action), Beth can perform the action $acc$ to accept  the offer or she can make a counteroffer and pass the turn to Alice. As soon as an agent accepts an offer, the negotiation ends and the pie is divided (\eg, states $q_3$, $q_6$, $q_9$, ans $q_{12}$).

\input{Sections/figure-negotiation}

Let us use the atomic propositions $\text{twothird}_\ag$, $\text{half}_\ag$, and $\text{onethird}_\ag$ to denote whether agent $\ag \in \{$Alice, Beth$\}$ has received two thirds, half, or one-third of the pie. Agents may have different preferences for how much of the pie they receive. Discounting functions can be used to capture the share they are more eager to receive. For instance, let 
 $$\psi_\ag \egdef \F_{\! \df_{2/3}} \, \text{twothird}_\ag \lor \F_{\! \df_{1/2}} \, \text{half}_\ag \lor \F_{\! \df_{1/3}} \, \text{onethird}_\ag$$
 be the goal of agents $\ag \in\{Alice, Beth\}$, 
with the discounting functions defined as  $\df_{n/m} \egdef \frac{n}{m}\df_{pie}(i)$ for $n,m\in\{1,2,3\}$. This goal stresses that agent $\ag$ prefers to get two-thirds of the pie over half or one-third, and half of the pie over one-third. Note that for the sake of simplicity of this example, deadlines are not considered in $\psi_\ag$. %

To continue the example, consider that the  discounting function $\df_{pie}$ is defined as follows 
\[
\df_{pie}(i) =
    \begin{dcases}
        1 &\text{if } i \leq 2\\
        \Big(\frac{1}{2}\Big)^i &\text{otherwise}
    \end{dcases}
\] 
This represents that the pie starts shrinking only after the 2nd game stage (states $q_9$, $q_{10}$, $q_{11}$ and so on). After that, the pie shrinks by half in each successive state. In this case, the rate in which the pie shrinks motivates agents to accept the first proposed division.  

Given the discount function $\df_{pie}(i)$ and the goals $\psi_{\text{Alice}}$ and $\psi_{\text{Beth}}$, a Nash equilibrium from the game is the strategy profile $(\strat_{\text{Alice}},\strat_{\text{Beth}})$, where $\strat_{\text{Alice}}$ and $\strat_{\text{Beth}}$
are strategies such that $\strat_{\text{Alice}}(q_0) = [\frac{2}{3},\frac{1}{3}]$ and $\strat_{\text{Beth}}(q) = acc $ for any state $q$. Thus, we have that  
$\semSLnegglobal{\hat{\varphi}_{\text{NE}}}\neq 0$. 
  

%% file: Sections/figure-secretary.tex
 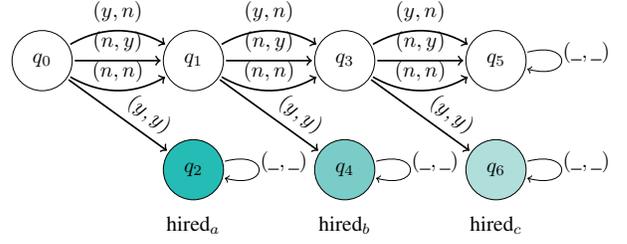
\begin{figure}
     \centering
     \scalebox{.8}{
     \begin{tikzpicture}[shorten >= 1pt, shorten <= 1pt, minimum size  = 1cm, auto]
  \tikzstyle{rond}=[circle,draw=black] 
 \tikzstyle{wina}=[fill=Emerald!70]
  \tikzstyle{winb}=[fill=Emerald!50]
   \tikzstyle{winc}=[fill=Emerald!30]
   \tikzstyle{label}=[sloped,shift={(0,-.2cm)}]
  
  \node[rond] (q0) {$q_0$};
  \node[rond] (q1) [ right = 1.5cm of q0] {$q_1$}; 
  \node[rond,wina] (q2) [below = 0.8cm of q1] {$q_2$};
  \node[rond] (q3) [right  = 1.5cm of q1] {$q_3$};
  \node[rond,winb] (q4) [below  = 0.8cm of q3] {$q_4$};
  \node[rond] (q5) [right  = 1.5cm of q3] {$q_5$};
  \node[rond,winc] (q6) [below = 0.8cm of q5] {$q_6$};

   \node (l1) [below = -.1cm of q2] {$\text{hired}_a$}; 
   \node (l2) [below = -.1cm of q4] {$\text{hired}_b$};
   \node (l3) [below = -.1cm of q6] {$\text{hired}_c$};
    

 \path[->,thick,bend right] (q0) edge  node [label,pos=.5]    {{$(n,n)$ }}    (q1);
 \path[->,thick,bend left] (q0) edge  node [label,pos=.5]    {{$(y,n)$ }}    (q1);
 \path[->,thick ] (q0) edge  node [label,pos=.5]    {{$(n,y)$ }}    (q1);
 \path[->,thick] (q0) edge  node  [label,pos=.7]   {{$(y,y)$}}    (q2);

  \path[->,thick,bend right] (q1) edge  node [label,pos=.5]    {{$(n,n)$ }}    (q3);
 \path[->,thick,bend left] (q1) edge  node [label,pos=.5]    {{$(y,n)$ }}    (q3);
 \path[->,thick ] (q1) edge  node [label,pos=.5]    {{$(n,y)$ }}    (q3);
 \path[->,thick] (q1) edge  node  [label,pos=.7]   {{$(y,y)$}}    (q4);

\path[->,thick,bend right] (q3) edge  node [label,pos=.5]    {{$(n,n)$ }}    (q5);
 \path[->,thick,bend left] (q3) edge  node [label,pos=.5]    {{$(y,n)$ }}    (q5);
 \path[->,thick ] (q3) edge  node [label,pos=.5]    {{$(n,y)$ }}    (q5);
 \path[->,thick] (q3) edge  node  [label,pos=.7]   {{$(y,y)$}}    (q6);

\path [->] (q2) edge[loop right]  node  [pos=.3]   {{$(\_,\_)$}}  (q2);
\path [->] (q4) edge[loop right]  node  [pos=.3]   {{$(\_,\_)$}}  (q4);
\path [->] (q6) edge[loop right]  node  [pos=.3]   {{$(\_,\_)$}}  (q6);
\path [->] (q5) edge[loop right]  node  [pos=.3]   {{$(\_,\_)$}}  (q5);
\end{tikzpicture}
}
     \caption{$\CGS_{sec}$ representing the secretary problem with three candidates ($a$, $b$ and $c$) and two voters (Ann and Bob). 
     In state $q_0$ (similarly, $q_1$ and $q_3$), Ann and Bob vote on whether to hire candidate $a$ (resp. $b$  and $c$).
     States $q_2$, $q_4$, and $q_6$ represent the situation in which candidate $a$, $b$ and $c$ were hired, respectively.
     } 
     \label{fig:sec}
 \end{figure}

%% file: Sections/figure-negotiation.tex
 \begin{figure}
     \centering
     \scalebox{.8}{
     \begin{tikzpicture}[shorten >= 1pt, shorten <= 1pt, minimum size  = 1cm, auto]
  \tikzstyle{rond}=[circle,draw=black] 
  \tikzstyle{wina}=[fill=Emerald!70]
  \tikzstyle{winb}=[fill=Emerald!30]

   \tikzstyle{label}=[sloped,shift={(0,-.2cm)}]
  
  \node[rond] (q0) {$q_0$};
  \node[rond] (q1) [above right = 1.40cm and 1.5cm  of q0] {$q_1$}; 
  \node[rond] (q2) [below right= 1.40cm and 1.5cm of q0] {$q_2$};
 
  \node[rond] (q4) [right  = 2cm of q1] {$q_4$};
  \node[rond,wina] (q3) [above =0.5cm of q4] {$q_3$};
  \node[rond] (q5) [below = 0.5cm of q4] {$q_5$};
  \node (blanck5) [right = 0.5cm of q5] {$\cdots$};

  \node[rond] (q7) [right  = 2cm of q2] {$q_7$};
\node[rond,wina] (q6) [above = 0.5cm of q7] {$q_6$};
  \node[rond] (q8) [below = 0.5cm of q7] {$q_8$};
  \node (blanck8) [right = 0.5cm of q8] {$\cdots$};

  \node[rond] (q10) [right = 2cm of q4] {$q_{10}$};
  \node[rond,winb] (q9) [above = 0.5cm  of q10] {$q_9$};
  \node[rond] (q11) [below = 0.5cm of q10] {$q_{11}$};
  \node (blanck10) [right = 0.5cm of q10] {$\cdots$};
  \node (blanck11) [right = 0.5cm of q11] {$\cdots$};

  \node[rond] (q13) [right = 2cm of q7] {$q_{13}$};
  \node[rond,winb] (q12) [above = 0.5cm  of q13] {$q_{12}$};
  \node[rond] (q14) [below = 0.5cm of q13] {$q_{14}$};
  \node (blanck14) [right = 0.5cm of q14] {$\cdots$};
  \node (blanck13) [right = 0.5cm of q13] {$\cdots$};

 \path[->,thick] (q0) edge  node [label,pos=.5]    {{$([\frac{1}{2},\frac{1}{2}],\_)$ }}    (q1); 
 \path[->,thick] (q0) edge  node [label,pos=.5]    {{$([\frac{2}{3},\frac{1}{3}],\_)$ }}    (q2);
 \path[->,thick] (q1) edge  node [label,pos=.5]    {{$(\_,acc)$ }}    (q3);
 \path[->,thick] (q1) edge  node [label,pos=.5]    {{$(\_, [\frac{1}{2},\frac{1}{2}])$ }}    (q4);
\path[->,thick] (q1) edge  node [label,pos=.5]    {{$(\_,[\frac{1}{3},\frac{2}{3}])$ }}    (q5);

 \path[->,thick] (q2) edge  node [label,pos=.5]    {{$(\_,acc)$ }}    (q6);
 \path[->,thick] (q2) edge  node [label,pos=.5]    {{$(\_, [\frac{1}{2},\frac{1}{2}])$ }}    (q7);
\path[->,thick] (q2) edge  node [label,pos=.5]    {{$(\_,[\frac{1}{3},\frac{2}{3}])$ }}    (q8);

 \path[->,thick] (q4) edge  node [label,pos=.5]    {{$(acc,\_)$ }}    (q9);
 \path[->,thick] (q4) edge  node [label,pos=.5]    {{$( [\frac{1}{2},\frac{1}{2}],\_)$ }}    (q10);
\path[->,thick] (q4) edge  node [label,pos=.5]    {{$([\frac{2}{3},\frac{1}{3}],\_)$ }}    (q11);

 \path[->,thick] (q7) edge  node [label,pos=.5]    {{$(acc,\_)$ }}    (q12);
 \path[->,thick] (q7) edge  node [label,pos=.5]    {{$( [\frac{1}{2},\frac{1}{2}],\_)$ }}    (q13);
\path[->,thick] (q7) edge  node [label,pos=.5]    {{$([\frac{2}{3},\frac{1}{3}],\_)$ }}    (q14);

\path[->,thick] (q5) edge  node [label,pos=.5]    {}    (blanck5);
\path[->,thick] (q8) edge  node [label,pos=.5]    {}    (blanck8);
\path[->,thick] (q10) edge  node [label,pos=.5]    {}    (blanck10);
\path[->,thick] (q11) edge  node [label,pos=.5]    {}    (blanck11);
 \path[->,thick] (q13) edge  node [label,pos=.5]    {}    (blanck13);
 \path[->,thick] (q14) edge  node [label,pos=.5]    {}    (blanck14);
 
\path [->] (q3) edge[loop right]  node  [pos=.3]   {{$(\_,\_)$}}  (q3);
\path [->] (q6) edge[loop right]  node  [pos=.3]   {{$(\_,\_)$}}  (q6); 
\path [->] (q9) edge[loop right]  node  [pos=.3]   {{$(\_,\_)$}}  (q9); 
\path [->] (q12) edge[loop right]  node  [pos=.3]   {{$(\_,\_)$}}  (q12); 
\end{tikzpicture}
}
     \caption{$\CGS_{ngt}$ representing the single-issue negotiation problem with  two agents, who alternate into proposing a division of the resource.   The negotiation ends when one of the agents agree with the proposed division (\eg, at the colored states $q_3$, $q_6$, $q_9$, $q_{12}$).  
     } 
     \label{fig:neg}
 \end{figure}
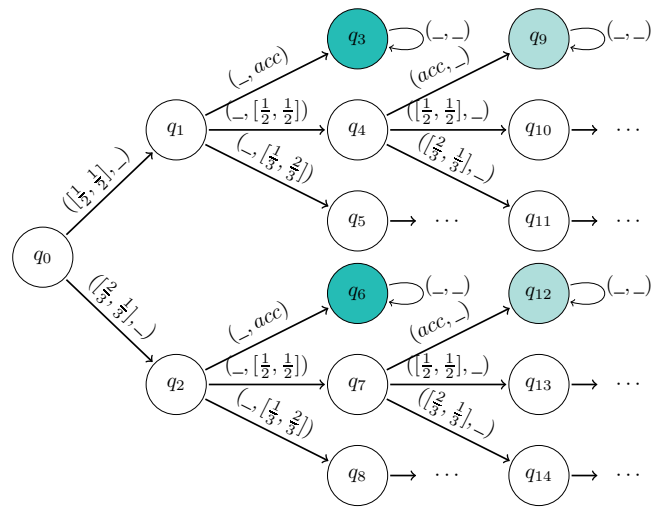

%% file: Sections/Model_checking.tex
\section{Model Checking \SL With Discounting}
\label{sec:mc}
In this section, we study the quantitative model checking problem for \SLD.  
Let us first define it formally.  

\begin{definition} 
The threshold \emph{model-checking problem} for \SLD\ consists in deciding, given a formula $\phi$, a CGS $\CGS$,  $\setting\in \{\perfectrecall,\memoryless\}$, and a threshold $\thresholdmc \in [0,1]$, whether $\semSLglobal{\setting}{\phi}\geq \thresholdmc$.
\end{definition}

\subsection{Memoryless Strategies}

Model-checking \SLD\ with memoryless agents is no harder than model-checking \LTL or classical \SL with memoryless agents.

\begin{theorem}
\label{theo-mc2}
Assuming that functions in $\Func$ can be computed in polynomial space, model checking \SLD\ with 
memoryless agents is \PSPACE-complete.
\end{theorem}

\begin{proof}
    The lower bound is inherited from \SL \cite{CERMAK2018588}\footnote{The proof provided in  \cite{CERMAK2018588} (Thm. 1) considers \SL\ with epistemic operators, but by carefully reading the proof one can notice  removing the operators and restricting to perfect information do not affect the complexity results. }, which is captured by \SLD.  
    For the upper bound, we first show that each recursive call only needs at most polynomial space. Most cases are treated analogously to the proof of Theorem 2 in \cite{SLKF_KR21}.  
We focus on the case $\semSL{\memoryless}{\assign}{\pos}{\phi_1\U_{\! \df} \phi_2}$.  Let $\iplay=\out(\pos,\assign)$. 
When evaluating a discounted operator on $\iplay$, one can restrict attention to two cases: either the satisfaction value of the formula goes below $\thresholdmc$, in which case this happens after a bounded prefix (with index $m \geq 0$), or the satisfaction value always remains above $\thresholdmc$, in which case we can replace the discounted operator with a Boolean one. This allows us to look only at a finite number of stages. 
In the first case, let $m \geq 0$ denote the first index in which the satisfaction value of the formula goes below $\thresholdmc$. 
Let $\varphi = \phi_1\U_\df\phi_2$, it follows that
\begin{align*}
    \semSL{\memoryless}{\assign}{\pos}{\phi} \!=\! \sup_{i \ge 0}     \min \! &
        \left
      (\!\df(i)\semSL{\memoryless}{\assign}{\iplay_{i}}{\phi_2}, 
      \!\min_{0\leq j <i} 
      \semSL{\memoryless}{\assign}{\iplay_{j}}{\phi_1}\!\right)\\
       = \max_{0\le i \le m}    \min \!&\left
      (\!\df(j)\semSL{\memoryless}{\assign}{\iplay_{i}}{\phi_2}, 
      \!\min_{0\leq j <i} 
      \semSL{\memoryless}{\assign}{\iplay_{j}}{\phi_1}\!\right)
\end{align*}
This can be computed by a while loop that increments $i$, computes 
$\semSL{\memoryless}{\assign}{\iplay_{i}}{\phi_2}$, 
 $\min_{0\leq j <i} 
      \semSL{\memoryless}{\assign}{\iplay_{j}}{\phi_1}$ and their minimum, records the result if it is bigger than the previous maximum, and stops upon reaching a position that has already been visited. 
This requires storing the current value of $\min_{0\leq j <i} 
      \semSL{\memoryless}{\assign}{\iplay_{j}}{\phi_1}$, the current maximum, and the list of positions already visited, which are at most $|\setpos|$.
The second case is treated as for Boolean until (see Appendix \ref{apx:memoryless} for more details).  
     
Next, the number of nested recursive calls is at most $|\phi|$, so the total space needed is bounded by $|\phi|$ times a polynomial in the size of the input, and is thus polynomial.
\end{proof}

\subsection{Perfect Recall}
Our solution to the problem of \SLD\ model checking for perfect recall applies the automata-theoretic approach \cite{thomas1990automata,vardi1986automata}. The solution opportunely combines the techniques used for model-checking in \cite{AlmagorBK14,MMPV14}. 
Let us recall relevant definitions from automata theory (see \cite{kupferman2000automata} for details). 

\paragraph{Alternating tree automata.}  
			An \emph{alternating tree automaton} (ATA) is a tuple
			$\auto = \langle \Sigma, \Delta, Q, \delta, q_{0}, \aleph\rangle$, where $\Sigma$, $\Delta$, and $Q$ are,
			respectively, non-empty finite sets of \emph{input symbols},
			\emph{directions}, and \emph{states}, $q_0 \in Q$ is an
			\emph{initial state}, $\aleph$ is an \emph{acceptance condition}, and $\delta : Q \times \Sigma \to {\cal B}^+(\Delta \times Q)
   $ is an \emph{alternating transition function} that maps each
			pair of states and input symbols to a positive Boolean combination on the
			set of propositions of the form $(d, q) \in \Delta\times Q$,
			called \emph{moves}.

A $\Sigma$-labeled tree is a pair $\langle T, v \rangle$ where $T$ is a tree and $V:T\to\Sigma$ maps each node of $T$ to a letter in $\Sigma$.
\paragraph{Run} A run of an ATA $\auto = \langle \Sigma, \Delta, Q, \delta, q_0, \aleph \rangle$ on a $\Sigma$-labeled
$\Delta$-tree $\tree = \langle T, v \rangle$ is a $(\Delta \times Q)$-tree $R$ such that for all nodes $x \in R$, where $x = \prod_{i=1}^n
(d_i, q_i)$ and $y= \prod^n_{i=1} d_i$ with  $n \in [0,\omega[$, it holds that (i) $y \in T$ and (ii) there is a set of moves $S \subseteq \Delta \times  Q$ with $S \models \delta(q_n, v(y))$ %
such that $x \cdot (d, q) \in R$ for all $(d, q) \in S$.

 \emph{Alternating parity tree automata} (APT) are alternating tree automata along with a \emph{parity acceptance
		condition} \cite{thomas2002automata}. We consider ATAs along with the parity acceptance condition (APT)
$\aleph = (F_1,\dots,F_k) \in (2^Q)^+$
with $F_1 \subseteq \dots \subseteq F_k = Q$. 
A nondeterministic parity tree automaton (NPT) is a special case of APT in which each conjunction in the transition function $\delta$ has exactly one move $(d, q)$ associated with each direction $d$.

\paragraph{APT Acceptance.} An APT $\auto = \langle \Sigma, \Delta, Q, \delta, q_{0}, \aleph\rangle$  accepts a $\Sigma$-labeled
$\Delta$-tree $\tree$ if and only if is there exists a run $R$ of $\auto$ on $\tree$ such that all its infinite branches satisfy the acceptance condition $\aleph$. 
 By $\mathcal{L}(\auto)$ we denote the language accepted by the APT $\auto$, that is, the set of trees $\tree$ accepted by $\auto$.  The emptiness problem for $\auto$ is to decide whether
$\mathcal{L}(\auto) = \emptyset$.

\subsubsection{From \SLD\ to APT}

We reuse the structure of the model-checking approach for \SL \cite{MMPV14}. Precisely, given a \CGS $\CGS$, a state $\pos$, and an \SL-sentence $\phi$, the procedure consists of building an NPT that is non-empty if $\phi$ is satisfied in $\CGS$ at state $\pos$ (Thm 5.8 \cite{MMPV14}).
As an intermediate step to obtain the NPT, the construction builds an APT $\auto$ that accepts a tree encoding of $\CGS$ containing the information on an assignment $\assign$ iff the \CGS satisfies the formula of interest for $\assign$.  
The NPT $\nauto$ is obtained by using an APT direction projection with distinguished direction $\pos$ to the APT $\auto$ (Thm 5.4 \cite{MMPV14}). 
The size of the APT $\auto$ is polynomial in the size of $\CGS$  and exponential in the number $k$ of alternations of strategy quantifiers. Then, building the NPT $\nauto$  and checking its emptiness requires an additional exponent on top of the number of alternations $k$, 
which leads to a final complexity 
$(k+1)$-$\EXPTIME$-complete (and $\Ptime$ in the size of the $\CGS$). 
For adapting this procedure to model checking of \SLD\ with perfect recall, we need to unpack and extend the construction of the  APT shown in Lemma 5.6 in \cite{MMPV14}, which we do here in the rest of this section.

We define a  translation for each  $\SLD$ formula $\phi$ to an APT $\auto$  that recognizes a tree encoding $\tree$ of a CGS $\CGS$, containing  the information on the assignment $\assign$ iff $\semSL{\perfectrecall}{\assign}{\pos_\init}{\psi}\geq\thresholdmc$. 

Defining the appropriate transition function for the $\auto$ follows the semantics of $\SLD$ in the expected manner. 
The transitions involving the discounting operators need a careful treatment, as discounting formulas can take infinitely many satisfaction values. As for \LTLD\ \cite{AlmagorBK14}, given a threshold $\thresholdmc$ and a computation $\pi$, when evaluating a discounted operator on $\pi$, one can restrict attention to two cases: either the satisfaction value of the formula goes below $\thresholdmc$, in which case this happens after a bounded prefix, or the satisfaction value always remains above $\thresholdmc$, in which case we can replace the discounted operator with a Boolean one. 

As for \cite{MMPV14}, we use the
concept of encoding for a \CGS assignment. First, let $\text{Val}_{\phi} \!\egdef\! \text{free}(\phi) \!\to\! \Act$.

\paragraph{Assignment-State Encoding. } Let $\CGS$ be a \CGS, $\pos \in \setpos$ be a state, and
$\assign$ be an assignment. Then, the assignment-encoding for $\assign$ is the $(\text{Val}_{\phi} \times \setpos)$-labeled $\setpos$-tree $\tree$, $\langle T, u\rangle$, such that $T$ is the set of histories $h$ of $\CGS$ given $\assign$ starting in $\pos$ and $u(h) \egdef (f,q)$, 
where    $q$ is the last state in $h$ and $f:\text{free}(\phi) \to \Act$ is defined by $f(\var)\egdef \assign(\var)(h)$ for each free variable $\var \in \text{free}(\psi)$.  

\begin{lemma}  
\label{lemma:translation}
Let $\CGS$ be a \CGS, $\phi$ an \SLD\ formula, and  $\thresholdmc\in [0,1]$ be a threshold.  Then, there exists an 
$\auto_{\phi,\thresholdmc} = \langle \text{Val}_{\phi} \times \setpos, \setpos, Q, \delta, q_{0}, \aleph\rangle$ such that, for all states $q\in Q$, and assignments $\assign$, it holds that $\semSL{\perfectrecall}{\assign}{\pos}{\phi}>\thresholdmc$ iff $\tree \in \mathcal{L}(\auto_{\phi,\thresholdmc})$, where $\tree$ is the assignment-state encoding for $\assign$.
\end{lemma}
\begin{proof}[Proof sketch]
    The construction of the APT $\auto_{\phi,\thresholdmc}$ is done recursively on the structure of the formula $\phi$.  
    Let $xcl(\phi)$ be the extended closure of $\phi$ defined analogously to \cite{AlmagorBK14}. 
    The state space $Q$ consists of two types of states. 
    Type-1 states are assertions of the form $(\psi>t)$ or $(\psi<t)$, where $\psi\in xcl(\phi)$ is not an \SL\ formula   
    and $t\in [0,1]$. Type-2 states correspond to $\SL$ formulas.
    The precise definition of $xcl(\phi)$, Type 1 and Type 2 states  is analogously to \cite{AlmagorBK14} and can be found in Appendix \ref{apx:perfectrecall}.
    Let $S$ be the set of Type-1 and Type-2 states for all $\psi\in xcl(\phi)$ and thresholds $t\in[0,1]$. Then, $Q$ is the subset of $S$ constructed on-the-fly according to the transition function defined below.

The transition function $\delta: (\text{Val}_{\phi} \times \setpos) \to {\cal B}^+(\setpos \times Q)$
is defined as follows. For Type-2 states, the transitions are as in the standard translation from $\SL$ to APT \cite{MMPV14}. For the other states, we define the transitions as follows. Let $(f, \pos)\in (\text{Val}_{\phi} \times \setpos)$ and $\oplus \in \{<,>\}$. 
\begin{itemize}

\item $\delta((p>t),(f, \pos))=\left[\begin{array}{ll}
true & \text{ if }p\in \val(\pos) \text{ and } t<1,\\
false & \text{ otherwise.}
\end{array}\right.$

\item
$\delta((p<t),(f, \pos))=\left[\begin{array}{ll}
false & \text{ if }p\in \val(\pos) \text { or } t=0,\\
true & \text{ otherwise.}
\end{array}\right.$ 

\item $\delta((\exists\var \psi) \oplus t), (f, \pos)) = \bigvee_{\act \in \Act} \delta'( \psi \oplus t, (f[\var \mapsto \act], \pos))$  
where $\delta_\psi'$ is obtained by  nondeterminizing the APT $\auto_{\psi,t}$, by applying the classic transformation \cite{MULLER1987267} which gives the equivalent NPT $\auton_{\psi, t} = \langle \text{Val}_{\psi} \times \setpos, \setpos, Q', \delta', q_{0}', \aleph'\rangle$.

\item $\delta(((\var,\ag)\psi \oplus t),(f, \pos)) = 
\delta'( (\psi \oplus t), (f', \pos))$ where $f' = f[\varb \mapsto f(\var)]$ if $\varb \in \text{free}(\psi)$, and $f'=f$ otherwise.

\end{itemize} 

The remaining cases are a simple adaptation of the proof in \cite{AlmagorBK14} (Thm 1) to the input symbols $\text{Val}_{\phi} \times \setpos$. We provide more details of the proof in Appendix \ref{apx:perfectrecall}. 

The initial state of $\auto_{\phi,\thresholdmc}$ is $(\phi>\thresholdmc)$. 
The accepting states are these of the form $(\psi_1 \U \psi_2 < t)$ for Type-1 states, as well as accepting states that arise in the standard translation of Boolean \SL to APT for Type-2 states.  
While the construction as described above is infinite 
only finitely many states are reachable from the initial state, and we can compute these
states in advance.  
\end{proof}
%

Using the threshold and the discounting behavior of the discounted-Until, we can restrict attention to a finite resolution of satisfaction values, enabling the construction of a finite automaton. 
Its size depends on the functions in $\Func$. 
Intuitively, the faster the discounting tends to 0, the fewer states there will be. 
Thus, the exact complexity of model checking \SLD\ (which relies on the size of the APT) depends on two aspects.
First, the alternation of quantifiers in the formula and, second, the type of discounting functions considered. In the specific setting where $\Func$ is composed of  exponential-discounting functions, (\ie, $\Func \subseteq \{\df(j)=\lambda^j : j \in (0, 1) \cap \mathbb{Q}\}$), the overall complexity remains as it is for \SL.  
Exponential discounting functions are perhaps the most common class of discounting functions, as they describe many natural processes (\eg, temperature change 
and effective interest rate 
\cite{shapley1953stochastic,alfaro2003discounting}).

\begin{theorem}
\label{theo-mc}
Assuming that  
functions in $\Func$ are exponential-discounting, 
model checking \SLD\ with 
memoryfull agents is $(k+1)$-\EXPTIME and $k$-\EXPSPACE w.r.t  the number $k$ of  quantifiers alternations in the specification.
\end{theorem}
\begin{proof}[Proof sketch]  
The model checking procedure from \cite{MMPV14} is  $(k+1)$-\EXPTIME-complete and $k$-\EXPSPACE  w.r.t the number $k$ of quantifiers alternations in the specification.  
Let $\thresholdmc \in (0,1)$ be a threshold. 
%
When discounting by an exponential-discounting function $\df(j)=\lambda^j  \in \Func$, the
number of states in the APT constructed as per Lemma \ref{lemma:translation} is proportional to the maximal number $j$ such that $\lambda^j < \thresholdmc$, 
 which is polynomial in the description length of $\thresholdmc$ and $\lambda$ \cite{AlmagorBK14}.  
\end{proof}


%% file: Sections/Conclusion.tex
\section{Conclusion and Discussion}
\label{sec:conclusion}
In this paper, we proposed Strategy Logic with discounting (\SLD), which contains an operator that captures the idea that the longer it takes to fulfill a requirement, the smaller the satisfaction value is.  
This work extends the research on temporal and strategic reasoning in Game Theory.  %
As advocated by \citeauthor{pauly2003logic} (\citeyear{pauly2003logic}), logics for strategic reasoning can have an important role in the specification and verification of game-theoretical problems and, in particular, related to Automated Mechanism Design (AMD). 
Indeed, recent works have proposed a new approach for AMD based on model-checking and synthesis from specifications in \SLF \cite{SLKF_KR21,ijcai2022p61}. Remarkably, \SLD\ provides less complicated machinery in relation to \SLF, as it is defined over classical concurrent game structures. More importantly, it brings a new dimension for reasoning about mechanisms that take into consideration how events are affected by how long in the future they occur. 

There are several interesting directions for future work, including considering synthesis from \SLD-specifications as well as the setting of imperfect information.  With  \SL already, imperfect information yields undecidability, but known tractable fragments exist~\cite{Berthon21,BLMR20}. We will investigate them in the case of \SLD.